%% file: dsn-dsml-full.tex
\documentclass[conference]{IEEEtran}
\IEEEoverridecommandlockouts
\usepackage{cite}
\usepackage{amsmath,amssymb,amsfonts}
\usepackage{algorithmic}
\usepackage{graphicx}
\usepackage{textcomp}
\usepackage{xcolor}
\def\BibTeX{{\rm B\kern-.05em{\sc i\kern-.025em b}\kern-.08em
    T\kern-.1667em\lower.7ex\hbox{E}\kern-.125emX}}

\usepackage{amsthm}
\usepackage{verbatim,enumitem,mdframed,mathrsfs}
\usepackage{subcaption}
\usepackage[ruled,vlined]{algorithm2e}
\usepackage{setspace}

\newtheorem*{theorem*}{\bfseries Theorem}
\newtheorem{theorem}{\bfseries Theorem}

\newtheorem{lemma}{\bfseries Lemma}
\newtheorem*{lemma*}{\bfseries Lemma}

\newtheorem{assumption}{\bfseries Assumption}

\providecommand{\iprod}[2]{\ensuremath{\left\langle #1,\,#2  \right\rangle}}
\providecommand{\norm}[1]{\ensuremath{\left\lVert#1\right\rVert }}
\providecommand{\mnorm}[1]{\ensuremath{\left\lvert#1\right\rvert}}

\def\E{\mathbb{E}}
\def\R{\mathbb{R}}
\def\D{\mathcal{D}}

\def\z{{\bf z}}
\def\M{\mathsf{M}}

\def\g{\mathfrak{g}}

\def\var{\mathsf{Var}}

\def\H{\mathcal{H}}
\def\B{\mathcal{B}}

\begin{document}

\title{Byzantine Fault-Tolerant Distributed Machine Learning using D-SGD and Norm-Based Comparative Gradient Elimination (CGE)
 \thanks{Research reported in this paper was sponsored in part by the Army Research Laboratory under Cooperative Agreement W911NF-17-2-0196, and by Fritz Fellowship from Georgetown University.}
}

\author{\IEEEauthorblockN{Nirupam Gupta}
\IEEEauthorblockA{\textit{EPFL} \\
Lausanne, Switzerland \\
nirupam115@gmail.com}
\and
\IEEEauthorblockN{Shuo Liu}
\IEEEauthorblockA{\textit{Georgetown University} \\
Washington, D.C., USA \\
sl1539@georgetown.edu}
\and
\IEEEauthorblockN{Nitin Vaidya}
\IEEEauthorblockA{\textit{Georgetown University} \\
Washington, D.C., USA \\
nv198@georgetown.edu}
}

\maketitle

\begin{abstract}
    This paper considers the Byzantine fault-tolerance problem in distributed stochastic gradient descent (D-SGD) method -- a popular algorithm for distributed multi-agent machine learning. In this problem, each agent samples data points independently from a certain data-generating distribution. In the fault-free case, the D-SGD method allows all the agents to learn a mathematical model best fitting the data collectively sampled by all agents. We consider the case when a fraction of agents may be Byzantine faulty. Such faulty agents may not follow a prescribed algorithm correctly, and may render traditional D-SGD method ineffective by sharing arbitrary incorrect stochastic gradients. We propose a norm-based gradient-filter, named {comparative gradient elimination} (CGE), that robustifies the D-SGD method against Byzantine agents. We show that the CGE gradient-filter guarantees fault-tolerance against a bounded fraction of Byzantine agents under standard stochastic assumptions, and is \textit{computationally simpler} compared to many existing gradient-filters such as multi-KRUM, geometric median-of-means, and the spectral filters. We empirically show, by simulating distributed learning on neural networks, that the fault-tolerance of CGE is comparable to that of existing gradient-filters. We also empirically show that \textit{exponential averaging} of stochastic gradients improves the fault-tolerance of a generic gradient-filter.

\end{abstract}



\input{intro}

\input{algo}
\input{stochastic_gradients}
\input{correctness}
\input{experiments-dsn-dsml}

\section{Summary}
In this paper, we have proposed a gradient-filter named comparative gradient elimination (CGE) to confer Byzantine fault-tolerate to distributed learning using the distributed stochastic gradient descent method. We have shown that our algorithm tolerates a bounded fraction of Byzantine faulty agents, under some standard assumptions. We have demonstrated through experiments the applicability of our algorithm to distributed learning of neural networks. We have empirically shown that the fault-tolerance of CGE gradient-filter is comparable to state-of-the-art gradient-filters, namely multi-KRUM, geometric median-of-means, and coordinate-wise trimmed mean gradient-filters. Finally, we have also illustrated the effectiveness of exponential averaging for improving upon the fault-tolerance of any given gradient-filter in the D-SGD method.

\bibliography{ref.bib}
\bibliographystyle{plain}

\input{conv_proof}

\end{document}

%% file: intro.tex
\section{Introduction} 
\label{sec:intro}
The problem of distributed multi-agent learning or {\em federated} learning has gained significant attention in recent years~\cite{boyd2011distributed, duchi2011dual, smith2017federated, yang2019federated}. 
In this problem, 
there are multiple machines or {\em agents} in the system each sampling data points locally and independently. The goal is to design {\em distributed algorithms} that allow the agents to compute or {\em learn} a common mathematical model that optimally fits the data points collectively sampled by all the agents. 
Most prior works in distributed learning consider a fault-free setting wherein all the agents are free from faults and follow a prescribed algorithm honestly. However, in practical distributed systems, some agents may be faulty~\cite{blanchard2017machine, chen2017distributed, alistarh2018byzantine, data2020byzantine, guerraoui2018hidden, su2016fault, sundaram2018distributed}. 

We consider a system with $n$ agents where up to $f$ agents are Byzantine faulty. The identity of the Byzantine agents is a priori unknown. Byzantine agents may collude and share arbitrary incorrect information with other non-faulty agents~\cite{lamport1982byzantine}. For instance, Byzantine agents may share information corresponding to {\em poisonous data points}; see~\cite{steinhardt2017certified} and references therein. In the presence of such faulty agents, a reasonable goal is to design a distributed algorithm that allows all the {\em non-faulty} agents to learn a mathematical model that optimally fits the data points only sampled by the non-faulty agents. 
A standard formulation of fault-tolerance in distributed learning is presented below.

{\bf Fault-tolerance in distributed learning:}
Each non-faulty agent samples data points independently and identically from a true data-generating distribution $\D$ over the $m$-dimensional real vector space $\R^m$. 
The non-faulty agents a priori fix a learning model $\Pi$, e.g., a neural network~\cite{bottou2018optimization}, characterized by $d$ real-valued parameters compactly denoted by a vector $w \in \R^d$. For a given parameter vector $w$, each data point $z \in \R^m$ incurs a {\em loss} defined by a real-valued {\em loss function} $\ell: (w, \, z) \mapsto \R$. We define the {\em non-faulty expected loss function}:
\begin{align}
    Q(w) = \underset{z \sim \D}{\E} \, \ell(w, \,z), \quad \forall w \in \R^d. \label{eqn:exp_loss}
\end{align}
The goal of a Byzantine fault-tolerant distributed learning algorithm is to allow the non-faulty agents to compute an optimal learning parameter $w^*$ that minimizes $Q(w)$, despite the presence of Byzantine faulty agents. 


{\bf System architecture:}
We consider a synchronous server-based architecture shown in Fig.~\ref{fig:algo}. The server is assumed trustworthy, but up to $f$ agents may be Byzantine faulty. The trusted server helps solve the distributed learning problem in coordination with the agents.


\begin{figure}[htb!]
\centering
\centering \includegraphics[width=0.3\textwidth]{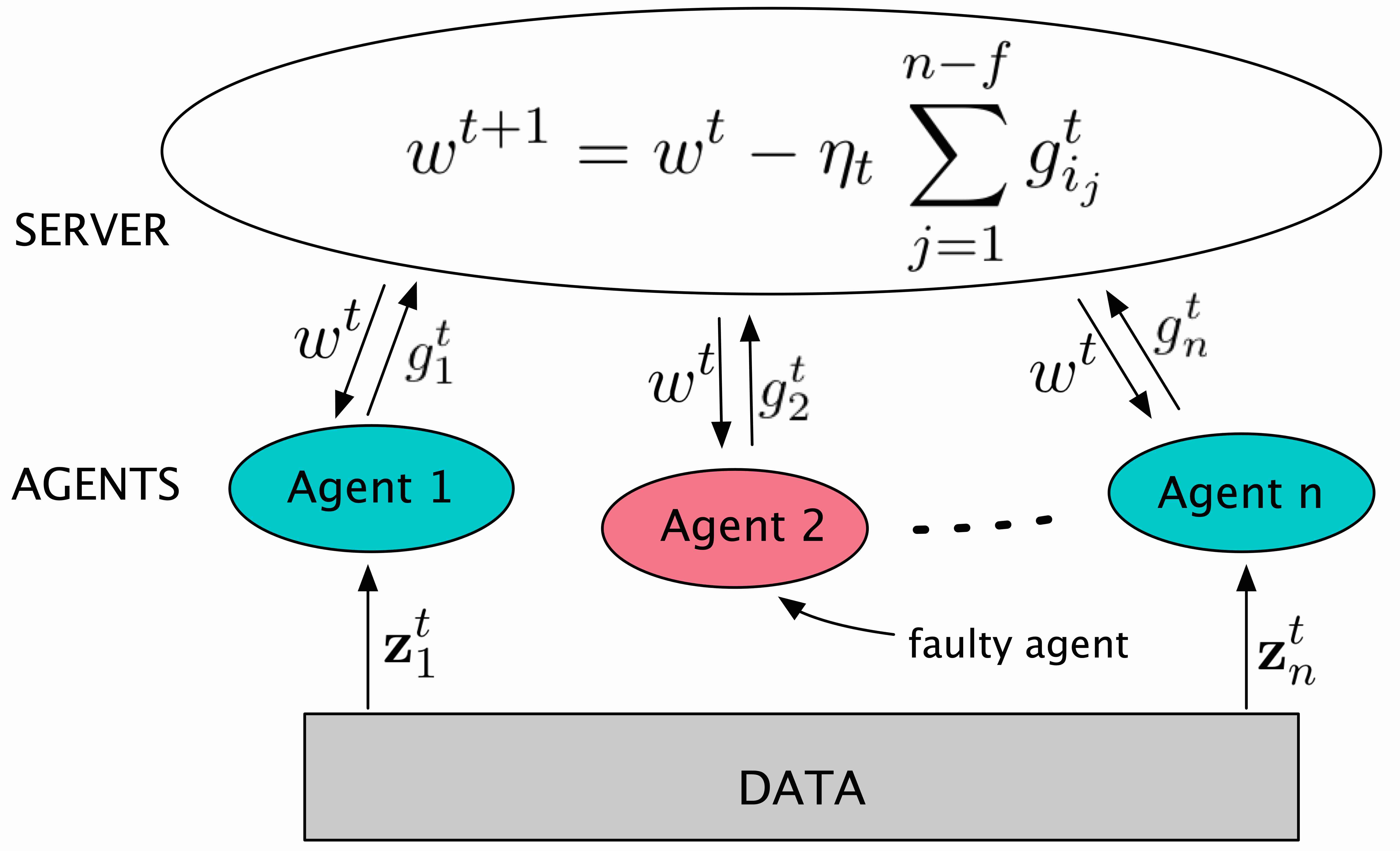} 
\caption{\footnotesize{\it The system architecture.}}
\label{fig:algo}
\end{figure}

{\bf Fault-tolerance in the distributed stochastic gradient descent (D-SGD) method:} We propose a fault-tolerance mechanism that confers fault-tolerance to 
the D-SGD method -- a standard distributed machine learning algorithm~\cite{bottou2018optimization}.
The D-SGD method is an iterative algorithm wherein the server maintains an estimate of an optimal learning parameter, which is updated iteratively using {\em stochastic} gradients computed by the different agents using i.i.d.~data points sampled from~$\D$, as shown in Figure~\ref{fig:algo}. In the fault-free setting, the D-SGD method converges to an optimal learning parameter, if the stochastic gradients have bounded variances~\cite{bottou2018optimization}. 
However, the traditional D-SGD method is rendered ineffective in the presence of Byzantine faulty agents that may send arbitrary incorrect stochastic gradients to the server~\cite{blanchard2017machine, chen2017distributed, guerraoui2018hidden}. 

Our proposed fault-tolerance mechanism relies on a norm-based {\em gradient-filter}, named {\em comparative gradient elimination} (CGE), that mitigates the detrimental impact of arbitrary incorrect stochastic gradients to the correctness of the D-SGD method. In our CGE gradient-filter, in each iteration, the server eliminates $f$ (out of $n$) received stochastic gradients with $f$ largest Euclidean norms. The current estimate in each iteration is updated using the average of the $n-f$ remaining stochastic gradients. Details of our algorithm, and its fault-tolerance property are presented later in Section~\ref{sec:algo}. Please refer Figure~\ref{fig:algo} for an illustration.

In contrast to the other previously introduced applications of norm-based gradient elimination solving other unrelated problems, e.g.,~\cite{pascanu2012understanding, shokri2015privacy}, our CGE gradient-filter employs an {\em adaptive} threshold. Specifically, in our case, the threshold for eliminating stochastic gradients is not a constant but varies depending upon the norms of the non-faulty agents' stochastic gradients. This difference is critical for the fault-tolerance property of the CGE gradient-filter.

We also incorporate an {\em exponential averaging} scheme where the server uses exponentially weighted averages of the agents' stochastic gradients, instead of their instantaneous values in the D-SGD method. We empirically show in Section~\ref{sub:appendix_averaging} that exponential averaging can notably improve the performance of a fault-tolerance mechanism such as the CGE gradient-filter.

\subsection{Summary of our contributions} 
\label{sub:contri}
We summarize below our main contributions and results.



\begin{itemize}
\setlength{\itemsep}{0.3em}
    \item {\bf A new algorithm and its fault-tolerance property:} We show that our algorithm, D-SGD method with CGE gradient-filter, guarantees Byzantine fault-tolerance under standard assumptions~\cite{bottou2018optimization}. Our result is summarized below, and presented formally in Section~\ref{sec:ft}. Notation $Pr$ denotes probability. \\
    
    \noindent 
    \fbox{\begin{minipage}{0.44\textwidth}
    \begin{theorem*}[Informal] Suppose that the variance $\sigma^2$ of the non-faulty stochastic gradients is bounded, and the loss function $Q(w)$ is $\lambda$-strongly convex with $\mu$-Lipschitz continuous gradients. Let $w^*$ denote a minimum of $Q(w)$ and $w^t$ denote an estimate of $w^*$ at the server in the $t$-th iteration of our algorithm. If $f/n$ is less than $\lambda/(2\lambda + \mu)$ then $\exists \rho \in (0, \, 1)$ and $M \in \Theta\left(\sigma\right)$ such that, $\forall \epsilon > 0$,
    \end{theorem*}
    \end{minipage}}
    \fbox{\begin{minipage}{0.44\textwidth}
    \begin{align*}
        \lim_{t \to \infty} Pr \left( \norm{w^t - w^*}^2 \leq \epsilon \right) \geq 1 - \frac{1}{\epsilon} ~ \left( \frac{\M^2}{1 - \rho}\right).
    \end{align*}
    \end{minipage}}
    Our algorithm is also effective in the heterogeneous data setting, where different agents may have different data distribution, provided that the agents' respective loss functions have some minimal redundancy, see~\cite{gupta2020fault_podc, gupta2019byzantine, liu2021approximate}. For simplicity, we only consider the homogeneous data setting in this paper.
    
    \item {\bf Exponential averaging for improved fault-tolerance:} We empirically show (for the first time) the efficacy of {\em exponential averaging} of stochastic gradients for improving Byzantine fault-tolerance of a gradient-filter.

    \item {\bf Empirical results:} We empirically show in Section~\ref{sec:exp} the efficacy of our algorithm for distributed learning on neural networks. We conduct experiments on a benchmark classification task, MNIST\cite{bottou1998online}, under varied fault types and fraction of faulty agents. 
        We show that the fault-tolerance of CGE gradient-filter is \textit{comparable} to the existing gradient-filters, namely multi-KRUM, geometric median-of-means, and coordinate-wise trimmed mean. 
\end{itemize}

\subsection{Comparisons with Related Work}
We present below key comparisons between CGE and other existing gradient-filters for fault-tolerance in D-SGD.

\begin{itemize}
\setlength{\itemsep}{0.3em}
    \item {\bf Computational simplicity:} The computational time complexity of CGE gradient-filter compares favourably to prominent gradient-filters, namely {\em multi-KRUM}~\cite{blanchard2017machine}, {\em geometric median-of-means}~\cite{chen2017distributed}, and  the {\em spectral gradient-filters}~\cite{charikar2017learning, data2020byzantine, diakonikolas2018sever, prasad2018robust}. In particular, the complexity of CGE gradient-filter is $\mathcal{O}(n (\log n + d))$, in comparison to $\mathcal{O}(n(n + d))$ of {multi-KRUM} and {geometric median-of-means}, and $\mathcal{O}(n d \min\{n, \, d\})$ of the {spectral} gradient-filters. We empirically show that despite its computational simplicity, the fault-tolerance of the CGE gradient-filter is comparable to these aforementioned gradient-filters. 
    \item {\bf Standard stochastic assumptions:} The coordinate-wise trimmed mean filter~\cite{yin2018byzantine}, the norm-based filter~\cite{ghosh2019communication}, and the signSGD filter~\cite{bernstein2018signsgd} have similar time complexity as CGE. However, the known fault-tolerance guarantees of these gradient-filters rely on strong assumptions about the non-faulty stochastic gradients that are uncommon in many learning problems~\cite{bottou2018optimization}. In particular,~\cite{ghosh2019communication, yin2018byzantine} assume non-faulty stochastic gradients to be {\em sub-exponential} random variables, and~\cite{bernstein2018signsgd} assumes non-faulty stochastic gradients to have a {\em unimodal symmetric} probability distributions. We, however, only assume the non-faulty stochastic gradients to have bounded variances -- an assumption that is common in all the prior works, and is required for the convergence of the D-SGD method even in the fault-free case~\cite{bottou2018optimization, bottou1998online}.
\end{itemize}

Other related works~\cite{li2019rsa, yang2019byrdie} consider the problem of Byzantine fault-tolerance in distributed learning using the {\em consensus optimization} methods~\cite{boyd2011distributed, nedic2009distributed}.



For obtaining the formal fault-tolerance property of our algorithm we assume the expected loss function $Q(w)$ to be strongly convex, unlike some of the aforementioned prior works. However, in learning problems when $Q(w)$ is convex, e.g.,~support vector machine and logistic regression, it is often regularized to a strongly convex function to mitigate overfitting~\cite{bottou2018optimization}. In many learning problems even if $Q(w)$ is not {\em globally} strongly convex, it is so in a neighborhood of local minimizers, thus our result can show the convergence of the D-SGD method with CGE gradient-filter in such regions of the search space.

%% file: algo.tex
\section{Proposed Algorithm}
\label{sec:algo}

In this section, we present our CGE gradient-filter for tolerating Byzantine faulty agents in  distributed learning using the D-SGD method.  The description of the algorithm below is followed by its fault-tolerance guarantee in Section~\ref{sec:ft}. 



\begin{algorithm}[htb!]
\SetAlgoLined
The server chooses the initial estimate $w^0$ arbitrarily from $\R^d$. Steps executed in each iteration $t$:
\\~\\

\noindent $\diamond ~$ {\bf Step S1}: The server broadcasts the current estimate $w^{t}$ to all the agents.
\\~\\

Each non-faulty agent $i$ sends to the server a stochastic gradient of the global expected loss function $Q(w)$ at $w^t$, i.e., a noisy estimator of the gradient $\nabla Q(w^t)$. A faulty agent may send an incorrect arbitrary vector for its stochastic gradient.
\\~\\

Let $g_i^t$ denote the gradient received by the server from agent $i$. If no gradient is from some agent $i$, $i$ must be faulty (since the system is assumed synchronous), and the server eliminates $i$ from the system.
%
\\~\\
    
\noindent $\diamond ~$ {\bf Step S2 (CGE gradient-filter):} The server sorts the $n$ received gradients as follows: 
\begin{align}
    \norm{g^t_{i_1}} \leq \ldots \leq \norm{ g^t_{i_{n-f}}} \leq \norm{ g^t_{i_{n-f+1}}} \leq \ldots \leq \norm{g^t_{i_{n}}}. \label{eqn:order}
\end{align}
where $g^t_{i_j}$, with $j$-th smallest norm, is from agent $i_j$. 

The server updates its current estimate using only $n-f$ stochastic gradients with smallest $n-f$ norms:
\begin{align}
    w^{t+1} = w^t - \eta_t \, \sum_{j = 1}^{n-f} g^t_{i_j} \label{eqn:algo_1}
\end{align}
where $\eta_t$ is the {\em step-size} of iteration $t$.

\caption{D-SGD with CGE Gradient-filter}
\label{alg:cge}
\end{algorithm}

Similar to the traditional D-SGD, the server maintains an estimate of an optimal learning parameter which is updated in each iteration using Algorithm~\ref{alg:cge}. For each iteration $t \in \{0, \, 1, \ldots\}$, let $w^t$ denote the estimate of the server. 
In \textbf{Step S1}, the server obtains from the agents their locally computed {\em stochastic gradients} of the expected loss function $Q(w)$ at $w^t$. There are multiple methods for computing stochastic gradients~\cite[Section 5]{bottou2018optimization}, one of which is described below. Note that a Byzantine faulty agent may send an arbitrary vector. In Step S2, to mitigate the detrimental impact of incorrect stochastic gradients,
the algorithm uses a filter to ``robustify" the gradient aggregation used for computing the updated estimate $w^{t+1}$. In particular, the server eliminates the stochastic gradients with the largest $f$ Euclidean norms, and uses the aggregate of the remaining $n-f$ stochastic gradients with $n-f$ smallest Euclidean norms to compute $w^{t+1}$, as shown in
Equation~\eqref{eqn:algo_1} below.
We refer to the method used in \textbf{Step S2} for elimination the largest
$f$ gradients as {\em Comparative Gradient Elimination} (CGE) gradient-filter, since the norms of the gradients are compared together to eliminate (or filter out) the gradients with the largest $f$ norms. 

%% file: stochastic_gradients.tex
To {compute a stochastic gradient} in an iteration $t$, a non-faulty agent $i$ samples
$k$ i.i.d.~data points $z^t_{i_1}, \ldots, \, z^t_{i_k}$ from the distribution $\D$. 
Then, for each non-faulty agent $i$, 
\begin{align}
    g^t_i = \frac{1}{k} \sum_{j = 1}^k \nabla \ell(w^t, \, z^t_{i_j}). \label{eqn:avg_grad}
\end{align}
$k$ is referred as the {\em data batch-size} or simply batch-size.

{\bf Complexity:} Note that computing the Euclidean norms of $n$ $d$-dimensional vector takes $\mathcal{O}(nd)$ time, and sorting of $n$ values takes $\mathcal{O}(n\log n)$ time. Thus, the time complexity of the CGE gradient-filter is $\mathcal{O}(n (d + \log n))$ in each iteration.



%% file: correctness.tex
\subsection{Fault-Tolerance Property}
\label{sec:ft}
In this section, we present a formal convergence result for our algorithm under standard assumptions that hold true in most learning problems~\cite{bottou2018optimization}. 

For each non-faulty agent $i$, let 
\begin{align}
    \z^t_i = \{z^t_{i_1}, \ldots, \, z^t_{i_k}\} \label{eqn:non-f_data}
\end{align}
denote the collection of $k$ i.i.d.~data points sampled by the agent $i$ in iteration $t$. Now, for each agent $i$ and iteration $t$ we define a random variable 
\begin{align}
    \zeta^t_i = \left\{ \begin{array}{ccc} \z^t_i & , & \text{agent $i$ is non-faulty} \\ g^t_i & , & \text{agent $i$ is faulty}\end{array}\right. \label{eqn:def_zeta_i}
\end{align}
Recall that~$g^t_i$ may be an arbitrary $d$-dimensional random variable for each Byzantine faulty agent $i$. For each iteration $t$, let
\begin{align}
    \zeta^t = \{\zeta^t_i, ~ i = 1, \ldots, \, n\}, \label{eqn:def_zeta}
\end{align}
and let $\E_t$ denote the expectation of a function of the collective random variables $\zeta^0, \ldots, \, \zeta^t$, given the initial estimate $w^0$. Specifically, 
\begin{align}
    \E_t(\cdot) = \E_{\zeta^0, \ldots, \, \zeta^{t}} (\cdot), \quad \forall t \geq 0. \label{eqn:notation_exp}
\end{align}

We make the following assumptions that are satisfies in many machine learning problems~\cite{bottou2018optimization, yin2018byzantine}.

\begin{assumption}[Bounded variance]
\label{asp:bnd_var}
Assume that there exists a finite real value $\sigma$ such that for all non-faulty agent $i$,
\[\E_{\zeta^t_i} \norm{ g^t_i - \E_{\zeta^t_i} \left(g^t_i\right)}^2 \leq \sigma^2, \quad \forall t.\]
\end{assumption}

\begin{assumption}[Lipschitz smoothness]
\label{asp:lip}
Assume that there exists a finite positive real value $\mu$ such that
\begin{align*}
    \norm{\nabla Q(w) - \nabla Q(w')} \leq \mu \norm{w - w'}, \quad \forall w, \, w' \in \R^d.
\end{align*}
\end{assumption}

\begin{assumption}[Strong convexity]
\label{asp:str_cvx}
Assume that there exists a finite positive real value $\lambda$ such that for all $w, \, w' \in \R^d$,
\begin{align*}
    \iprod{w - w'}{\nabla Q(w) - \nabla Q(w')} \geq \lambda \norm{w - w'}^2.
\end{align*}
\end{assumption}

We define a {\em fault-tolerance margin}  
\begin{align}
    \alpha = \frac{\lambda}{2\lambda + \mu} - \frac{f}{n} \label{eqn:alpha_learn}
\end{align}
that determines the maximum fraction of faulty agents $f/n$ tolerable by our algorithm. Lastly, we define an upper bound for the step-size $\eta_t$ in~\eqref{eqn:algo_1},
\begin{align}
    \overline{\eta} = \left(\frac{2(2 \lambda + \mu)n}{n^2 + (n-f)^2 \mu^2 } \right) \, \alpha. \label{eqn:step_learn}
\end{align}


Theorem~\ref{thm:conv} below presents a key fault-tolerance property of our algorithm, i.e., Algorithm~\ref{alg:cge}. Recall that $w^*$ denotes a minimum of the global expected loss function $Q(w)$, i.e.,
\[w^* \in \arg \min_{w \in \R^d} Q(w).\]

\begin{theorem} 
\label{thm:conv}
Consider Algorithm~\ref{alg:cge}. Suppose that the Assumptions~\ref{asp:bnd_var},~\ref{asp:lip}, and~\ref{asp:str_cvx} hold true. If the fault-tolerance margin $\alpha$ is positive, and in the update law~\eqref{eqn:algo_1} the step-size $\eta_t = \eta \in \left(0, ~ \overline{\eta} \right)$ for all $t$ then 
\begin{align}
    \rho  = 1 - \left(n^2 + (n-f)^2 \mu^2 \right) \, \eta \left( \overline{\eta} - \eta \right) \in (0, ~ 1), \label{eqn:rho}
\end{align}
and, for all $t \geq 1$, 
\begin{align}
    \E_{t-1} \norm{w^{t} - w^*}^2 \leq \rho^{t} \norm{w^{0} - w^*}^2 + \left(\frac{1 - \rho^{t}}{ 1- \rho}\right) \M^2 \label{eqn:rate}
\end{align}
where
\begin{align}
    \M^2 = \left(\frac{f^2 \left( 1 + \sqrt{n-f-1}\right)^2}{n^2} + \eta^2 (n-f)^2 \right) \sigma^2. \label{eqn:ss_error}
\end{align}
\end{theorem}

According to Theorem~\ref{thm:conv}, if $\alpha > 0$, i.e., 
\begin{align}
    \frac{f}{n} < \frac{\lambda}{2\lambda + \mu}, \label{eqn:bnd_fn}
\end{align}
then for small enough step-size, in~\eqref{eqn:algo_1}, our algorithm converges {\em linearly} to a neighborhood of a minimum of the global expected loss function~\eqref{eqn:exp_loss}. As $\rho < 1$,~\eqref{eqn:rate} implies that
\[\lim_{t \to \infty} \E_{t-1} \norm{w^{t} - w^*}^2 \leq \frac{\M^2}{ 1- \rho}. \]
Upon using the Markov's inequality, we obtain the probabilistic guarantee on training accuracy stated in Section~\ref{sub:contri}.

%% file: experiments-dsn-dsml.tex
\section{Experiments}
\label{sec:exp}

In this section, we present our key empirical results on fault-tolerance in distributed learning on neural networks using the D-SGD method with different gradient-filters. 
The fault-tolerance of different gradient-filters is evaluated through multiple experiments with varied fractions of faulty agents $f/n$, different types of faults and the data batch-size $k$.\\


We use multiple threads to simulate the distributed server-based system (ref. Fig.~\ref{fig:algo}), one for the server and others for agents. The inter-thread communication is handled through message passing interface. The simulator is built in Python using PyTorch~\cite{paszke2019pytorch} and MPI4py~\cite{dalcin2011parallel}, deployed on a Google Cloud Platform cluster 
with 64 vCPUs and 100 GB memory. 

We experiment on the dataset MNIST\cite{bottou1998online}, an image-classification dataset of handwritten digits comprising $60,000$ training and $10,000$ testing samples. We use a state-of-the-art neural network $\mathsf{LeNet}$ with of $431,080$ learning parameters. Thus, the value of dimension $d = 431,080$. 

We simulate a distributed system with $n = 40$ agents, among which $f$ faulty agents are chosen randomly. The server initiates the D-SGD method by choosing the initial estimate $w^0$, a $d$-dimensional vector, by uniform distributions near 0. The step-size is $\eta_t=0.01$ in every iteration $t$. To tolerate faulty agents, the server uses a gradient-filter as shown in Step S2 of Algorithm~\ref{alg:cge}. We compare the fault-tolerance of our CGE gradient-filter with the following prominent gradient-filters.
\begin{itemize}
    \item Geometric median (GeoMed);
    \item Geometric median-of-means (MoM)~\cite{chen2017distributed}, with group size $b = 2$;
    \item Coordinate-wise trimmed mean (CWTM)~\cite{yin2018byzantine};
    \item Multi-KRUM~\cite{blanchard2017machine}, with $m = 5$.
\end{itemize}

{\bf Types of Faults:} We simulate faulty agents that can exhibit two different types of faults listed below. The second one simulates inadvertently faulty agents that exhibit faulty behaviors due to hardware failures~\cite{he2020byzantinerobust}.

\begin{itemize}
    \item \textbf{Gradient-reverse fault}: A faulty agent sends to the server a vector directed opposite to its correct stochastic gradient with the same norm. Specifically, if $s^t_i$ denotes a correct stochastic gradient of faulty agent $i$ in iteration $t$ then agent $i$ sends to the server a vector $g^t_i = - s^t_i$. 
    \item \textbf{Label-flipping fault}: A faulty agent sends incorrect stochastic gradients due to erroneous output labels of its data points.
    Specifically, in our experiments with 10 different labels in MNIST, the original label of a data point $y$ sampled by a faulty agent is changed to $\widetilde{y}=9-y$.
\end{itemize}



\subsection{Results and analysis}

We now present our experimental results with the above setup under various settings specified below. 
We fix the number of faulty agents $f = 8$ out of $n = 40$ total agents, and the data batch-size $k = 64$. We do multiple experimental runs of the D-SGD method with different gradient-filters under different types of faults. The identity of the faulty agents is fixed across all the experiments. We also fix the random seeds used by agents for sampling data in training phase, so that across different experiments the same agent samples the same mini-batch of data for the same iteration. From our experiments, we observe that other than median-of-means, the three gradient-filters, including CGE, have comparable fault-tolerance against the two types of faults exhibited by the faulty agents. Median-of-means cannot tolerate \emph{gradient-reverse} faults, and exhibits a performance gap between other filters when facing \emph{label-flipping} faults. The plots for the losses and accuracies versus the number of iterations (or steps) are shown in Fig.~\ref{fig:fault-types-03} for the different experiments. Table~\ref{tab:run-time} shows the running time for each filter under \emph{label-flipping} faults. CGE has significantly smaller running time while preserving comparable performance.\\

\begin{figure}[tb!]
    \centering
    \includegraphics[width=.95\linewidth]{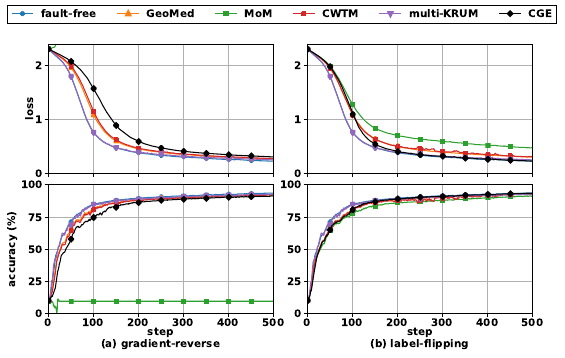}
    \caption{\footnotesize{\it Distributed learning of $\mathsf{LeNet}$ for MNIST using the D-SGD method with different gradient-filters (represented by different colors) to tolerate $f = 8$ faulty agents in the system. The training losses and the testing accuracies evaluated by the server after $0$ to $500$ iterations (or steps) of the different learning algorithms are plotted in the first and the second rows, respectively. Different columns contain the results for different types of faults simulated by the faulty agents. Here, the data batch-size $k=64$.}}
    \label{fig:fault-types-03}
\end{figure}

\begin{table}[tb!]
    \centering
    \caption{\footnotesize{Running time for distributed learning of \textsf{LeNet} for MNIST using D-SGD method with batch-size $k=64$ and the different gradient-filters in presence of $f=8$ faulty agents exhibiting \emph{label-flipping} type of fault, corresponding to the training process shown in Fig.~\ref{fig:fault-types-03}(b). Total running time refers to the time for 500 iterations.}}
    \footnotesize
    \begin{tabular}{c|cc}
        \textbf{Filter} & \textbf{Time per iteration} (s) & \textbf{Total running time} (s) \\
        \hline
        GeoMed & 2.473 & 1236.5 \\
        MoM & 1.150 & 575.1 \\
        CWTM & 0.903 & 451.8 \\
        Multi-KRUM & 2.225 & 1112.4 \\
        CGE & \textbf{0.573} & \textbf{286.3} \\
    \end{tabular}
\label{tab:run-time}
\end{table}

To evaluate the influence of individual agents' data batch-size on the fault-tolerance by CGE gradient-filter, we conduct experiments with four different batch-sizes: $k = 32, \, 64, \, 128$, and $256$. For these experiments all faulty agents exhibit \textit{label-flipping} faults. The average training losses and testing accuracies between the $475$-th and $500$-th iterations (or steps) for different batch-sizes are shown in Fig.~\ref{fig:batch-size-01} where different colors represent different numbers $f$ of faulty agents. A larger batch-size results in stochastic gradients with smaller variances (see~\eqref{eqn:avg_grad} in Section~\ref{sec:algo}), and thus, as expected from our theoretical results in Section~\ref{sec:ft}, the fault-tolerance of CGE gradient-filter improves with increased batch-size. However, an increase in batch-size also increases the costs for computing the stochastic gradients in each iteration. We later present a {\em exponential averaging} technique that reduces the variances of the stochastic gradients, and improves the fault-tolerance of the CGE gradient-filter, without increasing batch-sizes or the costs for computing stochastic gradients.\\

\begin{figure}[tb!]
    \centering
    \includegraphics[width=.95\linewidth]{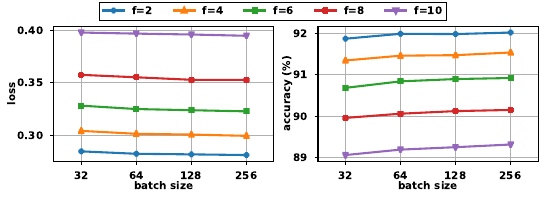}
    \caption{\footnotesize{\it Average training losses and testing accuracies evaluated between $475$ to $500$ steps of the D-SGD method with CGE gradient-filter for distributed learning of $\mathsf{LeNet}$ in the presence of different number of faulty agents $f$ (represented by different colors) each exhibiting the \emph{gradient-reverse} Byzantine fault, for different data batch-sizes.}}
    \label{fig:batch-size-01}
\end{figure}

\begin{figure}[tb!]
    \centering
    \includegraphics[width=.95\linewidth]{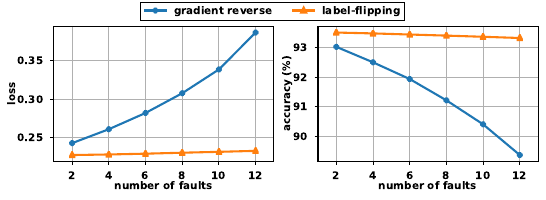}
    \caption{\footnotesize{\it Average training losses and testing accuracies evaluated between $475$ to $500$ steps of the D-SGD method with CGE gradient-filter for distributed learning of $\mathsf{LeNet}$ in the presence of varied number of faulty agents $f$. Different colors represent different types of faults. The data batch-size $k=64$. }}
    \label{fig:num-faults-01}
\end{figure}

Lastly, to evaluate the effect of the fraction of faulty agents $f/n$ on the fault-tolerance of CGE gradient-filter we conduct experiments for different values of $f$. In these experiments we set batch-size $k = 64$. The trend of average losses and accuracies observed between the $475$-th to $500$-th iterations (or steps) is shown in Fig.~\ref{fig:num-faults-01}. As expected from our theoretical results in Section~\ref{sec:ft}, the fault-tolerance of the CGE gradient-filter deteriorates with increase in the fraction of faulty agents.

\subsection{Exponential averaging of stochastic gradients}
\label{sub:appendix_averaging}

We observe from our experiments above, specifically plots in Fig.~\ref{fig:batch-size-01}, that fault-tolerance of CGE gradient-filter improves with increase in batch-size $k$. The reason why this happens is the fact that larger batch-size results in stochastic gradients with smaller variances $\sigma^2$, which, owing to our results in Section~\ref{sec:ft}, results in improved fault-tolerance. However, increase in batch-size also increases the cost of computing stochastic gradients in each iteration. Motivated from this observation, we propose a more economical technique, \textit{exponential averaging}, allowing non-faulty agents to compute stochastic gradients with reduced variances without increasing batch-size and their cost for computing stochastic gradients. Alternately, we may also use other existing variance reduction techniques from the stochastic optimization literature~\cite{bottou2018optimization}.


\begin{figure}[htb!]
    \centering
    \includegraphics[width=.95\linewidth]{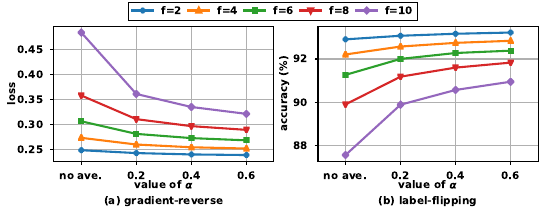}
    \caption{\footnotesize{\it Average training losses and testing accuracies evaluated between $475$ to $500$ steps of the D-SGD method with CGE gradient-filter for distributed learning of $\mathsf{LeNet}$ in the presence of different number of faulty agents $f$ (represented by different colors) when applying \emph{exponential averaging} of stochastic gradients for different values of $\beta$. \emph{no ave.} indicates the case when no averaging is not used. Here, the data batch-size $k = 64$, and the fault type is \emph{norm-confusing}.}}
    \label{fig:averaging-01}
\end{figure}

\begin{figure}[tb!]
    \centering
    \includegraphics[width=.95\linewidth]{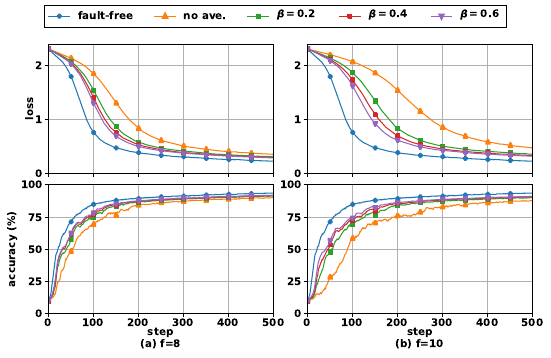}
    \caption{\footnotesize{\it Distributed learning of $\mathsf{LeNet}$ for MNIST using the D-SGD method with batch-size $k=64$ and CGE gradient-filter in the presence of $f = 8$ and $10$ faulty agents exhibiting \emph{norm-confusing} type of faults. \emph{Exponential averaging} of stochastic gradients applied with different values of $\beta$  (represented using different colors).}}
    \label{fig:averaging-detail}
\end{figure}

For each iteration $t$ and agent $i$, the server maintains an {\bf exponentially weighted average} $h^t_i$ of the stochastic gradients received from agent $i$ so far. Specifically, for $\beta \in [0, \, 1)$,
\begin{align}
    h^t_i = \beta\, h^{t-1}_i + (1-\beta) \, g^t_i. \label{eqn:grad_avg_scheme}
\end{align}
where $g^t_i$ denotes the stochastic gradient received by the server from agent $i$ in iteration $t$. Given a set of $n$ vectors $y_1, \ldots, \, y_n$, let $\mathsf{CGE}_f\{y_1, \ldots, \, y_n\}$ denote the output of our CGE gradient-filter defined in step {S2} of Algorithm~\ref{alg:cge}. 
For each iteration $t$, in step {S2} of Algorithm~\ref{alg:cge} the server updates the current estimate $w^t$ to
\(w^{t+1} = w^{t} + \eta_t\cdot\mathsf{CGE}_f\left\{h^t_1, \ldots, \, h^t_n \right\}\).
It should be noted that the above averaging scheme does not increase the per iteration computation cost for an individual agent, unlike the case when we increase the data batch-size.\\



To evaluate the impact of the above exponential averaging on the fault-tolerance of CGE gradient-filter, we introduce a new type of fault designed deliberately against our CGE filter:

\noindent \textbf{Norm-confusing fault}: A faulty agent sends to the server a vector directed opposite to its correct stochastic gradient. However, different form \textit{gradient-reverse} fault, the norm of the vector is scaled to the $(f+1)$-th largest norm amongst the stochastic gradients of all the $n-f$ non-faulty agents. 

Experiments are conducted on the distributed learning of $\mathsf{LeNet}$ for MNIST with $k = 64$ and different values of $f$. The outcomes are shown in Fig.~\ref{fig:averaging-01}. Smaller training loss can be achieved with exponential averaging and larger value of $\beta$ in the same number of steps. Fig.~\ref{fig:averaging-detail} further shows that even if there are large number $f$ of faulty agents and plain CGE converges slowly, exponential averaging can still significantly improve the performance. As shown in Table~\ref{tab:ave-run-time}, exponential averaging does not have a significant impact on running time.

Experiments also show that exponential averaging improves the fault-tolerance of other gradient-filters. For example, Fig.~\ref{fig:averaging-detail-cwtm} shows the outcome of the distributed learning task with $k=64$ and $f=10$, using median-of-means and coordinate-wise trimmed mean filters. With exponential averaging and larger value of $\beta$, faster convergence can be achieved for both filters.

\begin{table}[tb!]
    \centering
    \caption{\footnotesize{Running time for distributed learning of \textsf{LeNet} for MNIST using D-SGD method with batch-size $k=64$ and the different gradient-filters in presence of $f=8$ faulty agents exhibiting \emph{norm-confusing} type of fault with or without applying \emph{exponential averaging} for different values of $\beta$, corresponding to the training process shown in Fig.~\ref{fig:averaging-detail}(a). Total running time refers to the time for 500 iterations.}}
    \footnotesize
    \begin{tabular}{c|cc}
        \textbf{Filter} & \textbf{Time per iteration} (s) & \textbf{Total running time} (s) \\
        \hline
        CGE (No ave.) & 0.613 & 306.9 \\
        CGE ($\beta=0.2$) & 0.733 & 366.7 \\
        CGE ($\beta=0.4$) & 0.731 & 365.5 \\
        CGE ($\beta=0.6$) & 0.733 & 366.7 \\
    \end{tabular}
\label{tab:ave-run-time}
\end{table}

\begin{figure}[tb!]
    \centering
    \includegraphics[width=.95\linewidth]{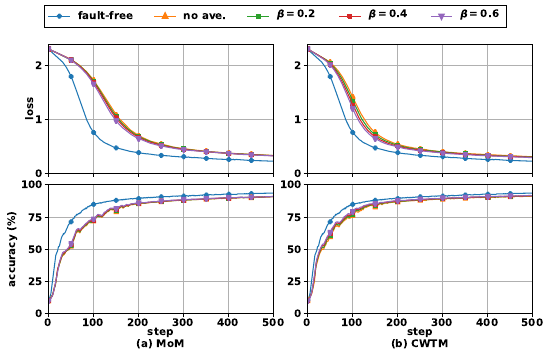}
    \caption{\footnotesize{\it Distributed learning of $\mathsf{LeNet}$ for MNIST dataset using the D-SGD method with batch-size $k=64$ and different gradient-filters in the presence of $f = 10$ faulty agents exhibiting \emph{norm-confusing} type of faults. \emph{Exponential  averaging} of stochastic gradients applied with different values of $\beta$  (represented using different colors). Filters used here are \emph{(a)} median-of-means (\emph{MoM}), and \emph{(b)} coordinate-wise trimmed mean (\emph{CWTM}), the two that has relatively poor performance without exponential averaging.}}
    \label{fig:averaging-detail-cwtm}
\end{figure}




%% file: conv_proof.tex
\section{Appendix: Proof of Theorem~\ref{thm:conv}}
\label{app:pf_thm}

Before we present our proof of the theorem we state below some basic results that are used later in the proof.

\subsection{Useful Observations and Lemmas}
Recall from~\eqref{eqn:def_zeta} in Section~\ref{sec:ft} that for each non-faulty agent $i$, and a deterministic real-valued function $\Psi$,
\begin{align}
    \E_{\zeta^t} \, \Psi\left(g^t_i\right) = \E_{\zeta^t_1, \ldots, \, \zeta^t_n} \, \Psi\left(g^t_i\right). \label{eqn:lem_exp_psi_1}
\end{align}
Also, from~\eqref{eqn:def_zeta_i}, recall that for each non-faulty agent $i$,
\begin{align}
    \zeta_i^t = \z^t_i. \label{eqn:lem_zeta_H}
\end{align}
For a given the current estimate $w^t$, the stochastic gradient $g^t_i$ is a function of data points $\z^t_i$ sampled by the agent $i$. As the non-faulty agents choose their data points independently and identically from distribution $\D$ in each iteration,~\eqref{eqn:lem_zeta_H} implies that for each non-faulty agent $i$,
\begin{align}
    \E_{\zeta^t_1, \ldots, \, \zeta^t_n} \Psi\left(g^t_i\right) = \E_{\z^t_i} \Psi\left(g^t_i\right). \label{eqn:lem_exp_psi_2}
\end{align}
Upon substituting from~\eqref{eqn:lem_exp_psi_2} in~\eqref{eqn:lem_exp_psi_1} we obtain that for each non-faulty agent $i$,
\begin{align}
    \E_{\zeta^t} \Psi\left(g^t_i\right)  = \E_{\z^t_i} \Psi\left(g^t_i\right). \label{eqn:thm_exp_psi_3}
\end{align}
For an arbitrary non-faulty agent $i$ and $t$, $\E_{\zeta^t_i}\left( g^t_i \right) = \E_{\z^t_i}\left( g^t_i \right)$.
Upon substituting $g^t_i$ from~\eqref{eqn:avg_grad} we obtain that
\begin{align}
    \E_{\zeta^t_i}\left( g^t_i \right) = \left(\frac{1}{k}\right) \, \E_{\z^t}\sum_{j = 1}^k \left( \nabla \ell (w^t, \, z^t_{i_j}) \right) \label{eqn:exp_grad_0}
\end{align}
where the gradient of loss function $\ell(\cdot, \, \cdot)$ is with respect to its first argument $w$. From~\eqref{eqn:non-f_data}, recall that $\z^t$ constitutes $k$ data points that are i.i.d.~as per the probability distribution $\D$. Upon using this fact in~\eqref{eqn:exp_grad_0} we obtain that 
\begin{align}
    \E_{\zeta^t_i}\left( g^t_i \right) = \frac{1}{k} \sum_{j = 1}^k \E_{z^t_{i_j} \sim \D} \left( \nabla \ell (w^t, \, z^t_{i_j}) \right), \quad \forall t. \label{eqn:exp_grad_1}
\end{align}
Note that
\begin{align}
    \nabla Q(w) = \E_{z \sim \D} \left( \nabla \ell (w, \, z) \right), \quad \forall w \in \R^d. \label{eqn:nabla_Q}
\end{align}
Substituting from~\eqref{eqn:nabla_Q} in~\eqref{eqn:exp_grad_1} we obtain that an arbitrary non-faulty agent $i$,
\begin{align}
    \E_{\zeta^t_i}\left( g^t_i \right) = \frac{1}{k} \sum_{j = 1}^k \left(\nabla Q(w^t) \right) = \nabla Q(w^t). \label{eqn:exp_grad}
\end{align}

Recall, from Assumption~\ref{asp:bnd_var}, that the variance of each non-faulty agent's stochastic gradient is bounded by $\sigma^2$.

\begin{lemma}
\label{lem:bnd_norm}
For an arbitrary iteration $t$, if Assumption~\ref{asp:bnd_var} holds true then for each non-faulty agent $i$,
\[ \E_{\zeta^t} \norm{ g^t_i }^2 \leq \sigma^2 + \norm{\nabla Q(w^t)}^2.\]
\end{lemma}

\begin{proof}
Let $i$ be an arbitrary non-faulty agent. Using the definition of Euclidean norm, note that for each iteration $t$,
\begin{align}
    \norm{g^t_i - \E_{\zeta^t} \left(g^t_i\right)}^2 = \norm{g^t_i}^2 - 2 \iprod{g^t_i}{\E_{\zeta^t} \left(g^t_i\right)} + \norm{\E_{\zeta^t} \left(g^t_i\right)}^2. \label{eqn:norm_expand}
\end{align}
As the expected value of a constant is the constant itself, upon taking expectations on both sides in~\eqref{eqn:norm_expand} we obtain that
\begin{align}
    \E_{\zeta^t} \norm{g^t_i - \E_{\zeta^t} \left(g^t_i\right)}^2 = \E_{\zeta^t} \norm{g^t_i}^2 - \norm{\E_{\zeta^t} \left(g^t_i\right)}^2. \label{eqn:lem_1_zeta_1}
\end{align}
Note, from~\eqref{eqn:thm_exp_psi_3}, that $\E_{\zeta^t} \left(g^t_i\right) = \E_{\z^t_i} \left(g^t_i\right)$, and $\E_{\zeta^t} \norm{g^t_i - \E_{\zeta^t} \left(g^t_i\right)}^2 = \E_{\z^t_i} \norm{g^t_i - \E_{\z^t_i} \left(g^t_i\right)}^2$.
Substituting these in~\eqref{eqn:lem_1_zeta_1} we obtain that
\begin{align}
    \E_{\z^t_i} \norm{g^t_i - \E_{\z^t_i} \left(g^t_i\right)}^2 = \E_{\zeta^t} \norm{g^t_i}^2 - \norm{\E_{\z^t_i} \left(g^t_i\right)}^2. \label{eqn:lem_1_zeta_2}
\end{align}
Recall from~\eqref{eqn:exp_grad} that $\E_{\z^t_i} ~ ( g^t_i ) = \nabla Q(w^t)$. Substituting this above we obtain that
\begin{align}
   \E_{\z^t_i} \norm{g^t_i - \E_{\z^t_i} \left(g^t_i\right)}^2 = \E_{\zeta^t} \norm{g^t_i}^2 - \norm{\nabla Q(w^t)}^2. \label{eqn:lem_1_zeta_3}
\end{align}
As Assumption~\ref{asp:bnd_var} holds true, $\E_{\z^t_i} \norm{g^t_i - \E_{\z^t_i} \left(g^t_i\right)}^2 \leq \sigma^2$.
Substituting this in~\eqref{eqn:lem_1_zeta_3} proves the lemma.
\end{proof}

Let $\mnorm{\cdot}$ denote the set cardinality. 
Recall that there are at least $n-f$ non-faulty agents. 
We define a set $\H$ constituting of $n-f$ non-faulty agents, i.e., if $i \in \H$ then agent $i$ is non-faulty and $\mnorm{\H} = n-f$. Let $\B = \{1, \ldots, \, n\} \setminus \H$ denote the remaining $f$ agents, some of which may be non-faulty. 

\begin{lemma}
\label{lem:order_stat}
For each iteration $t$, let $\nu_t$ denote the non-faulty agent in $\H$ with stochastic gradient $g^t_{\nu_t}$ of largest Euclidean norm, that is,
\[\norm{g^t_{\nu_t}} \geq \norm{g^t_i}, \quad \forall i \in \H.\]
For an arbitrary iteration $t$, if Assumption~\ref{asp:bnd_var} holds true then
\begin{align}
    \E_{\zeta^t} \, \norm{g^t_{\nu_t}} \leq \sigma \left( 1 + \sqrt{n-f-1} \right) + \norm{\nabla Q(w^t)}. \label{eqn:order_exp}
\end{align}
\end{lemma}
\begin{proof}
We begin our proof by reviewing a generic result on the upper bounds on the expectation of highest order statistic~\cite{arnold1979bounds, bertsimas2006tight}. For a positive finite integer $p$, let $R_1, \ldots, \, R_p$ be $p$ independent real-valued random variables. Consider a random variable $R_\nu = \max\{R_1, \ldots, \, R_p\}$.
Let $\E(\cdot)$ denote the mean value of a random variable. If the mean and the variance of the random variables $R_1, \ldots, \, R_p$ are identically equal to $\E(R)$ and $\var(R)$, respectively, then (see~\cite{arnold1979bounds})
\begin{align}
    \E(R_\nu) \leq \E(R) + \sqrt{\var(R) \left(p-1\right)}. \label{eqn:arnold_order_bnd}
\end{align}
~

Consider an arbitrary iteration $t$. Recall that $\H$ comprises of only non-faulty agents, specifically $n-f$ non-faulty agents. Thus, from~\eqref{eqn:thm_exp_psi_3} we obtain that, for all $i \in \H$,
\begin{align}
    \E_{\zeta^t} \norm{g^t_i}  &= \E_{\z^t_i} \norm{g^t_i}, \text{and} \label{eqn:lem_exp_t_3} \\
    \E_{\zeta^t}\left(\norm{g^t_i} -  \E_{\zeta^t}\norm{g^t_i} \right)^2 &= \E_{\z^t_i }\left( \norm{g^t_i} - \E_{\z^t_i}\norm{g^t_i} \right)^2. \label{eqn:lem_var_t_1}
\end{align}
Now, recall from the definition of $\nu_t$ and $g^t_{\nu_t}$ that $\norm{g^t_{\nu_t}}$ is a real-valued random variable such that
\[ \norm{g^t_{\nu_t}} = \max\left\{\norm{g^t_i}, ~ i \in \H\right\}.\]
Therefore, substituting from~\eqref{eqn:lem_exp_t_3} and~\eqref{eqn:lem_var_t_1} in~\eqref{eqn:arnold_order_bnd}, we obtain that for an arbitrary agent $i \in \H$,
\begin{align}
    \E_{\zeta^t} \, \norm{g^t_{\nu_t}} & \leq \E_{\z^t_i}\norm{g^t_i} \nonumber \\
    & + \sqrt{\E_{\z^t_i}\left(\norm{g^t_i} -  \E_{\z^t_i}\norm{g^t_i} \right)^2 \left( n-f-1\right)}. \label{eqn:lem_bnd_nu_1}
\end{align}
Owing to Jensen's inequality~\cite{boyd2004convex}, for all $i \in \H$,
\begin{align}
    \E_{\z^t_i} \, \norm{g^t_i} = \E_{\z^t_i} \, \sqrt{\norm{g^t_i}^2} \leq \sqrt{\E_{\z^t_i} \norm{g^t_i}^2}. \label{eqn:jen_1}
\end{align}
As Assumption~\ref{asp:bnd_var} holds true, from Lemma~\ref{lem:bnd_norm} we obtain that
\begin{align}
    \E_{\z^t_i} \, \norm{g^t_i}^2 \leq \sigma^2 + \norm{\nabla Q(w^t)}^2, \quad \forall i \in \H. \label{eqn:hon_bnd_1}
\end{align}
Substituting from~\eqref{eqn:hon_bnd_1} in~\eqref{eqn:jen_1} we obtain that
\begin{align}
    \E_{\z^t_i } \, \norm{g^t_i} \leq \sqrt{\sigma^2 + \norm{\nabla Q(w^t)}^2} ~ , \quad \forall i \in \H. \label{eqn:exp_bnd}
\end{align}
From triangle inequality, $\sqrt{\sigma^2 + \norm{\nabla Q(w^t)}^2} \leq \sigma + \norm{\nabla Q(w^t)}$. Substituting this in~\eqref{eqn:exp_bnd} proves that
\begin{align}
    \E_{\z^t_i} \, \norm{g^t_i} \leq \sigma + \norm{\nabla Q(w^t)} ~ , \quad \forall i \in \H. \label{eqn:exp_bnd_2}
\end{align}

Now, note that for all $i$,
\begin{align}
    \E_{\z^t_i}\left(\norm{g^t_i} -  \E_{\z^t_i }\norm{g^t_i} \right)^2 = \E_{\z^t_i } \norm{g^t_i}^2 - \left( \E_{\z^t_i } \norm{g^t_i} \right)^2. \label{eqn:bnd_basic_var}
\end{align}
As the Euclidean norm $\norm{\cdot}$ is a convex function~\cite{boyd2004convex}, Jensen's inequality implies that
\begin{align}
    \E_{\z^t_i } \norm{g^t_i} \geq \norm{\E_{\z^t_i }\left(g^t_i\right)}, \quad \forall i \in \H. \label{eqn:jen_norm_1}
\end{align}
Recall from~\eqref{eqn:exp_grad} that $\E_{\z^t_i }\left(g^t_i\right) = \nabla Q(w^t)$ for all $i \in \H$.
Upon substituting this in~\eqref{eqn:jen_norm_1} we obtain that
\begin{align}
    \E_{\z^t_i } \norm{g^t_i} \geq \norm{\nabla Q(w^t)}, \quad \forall i \in \H. \label{eqn:jen_norm_2}
\end{align}
Substituting from~\eqref{eqn:jen_norm_2} in~\eqref{eqn:bnd_basic_var} we obtain that for all $i \in \H$,
\begin{align}
    \E_{\z^t_i }\left(\norm{g^t_i} -  \E_{\z^t_i }\norm{g^t_i} \right)^2 \leq \E_{\z^t_i } \norm{g^t_i}^2 - \norm{\nabla Q(w^t)}^2. \label{eqn:bnd_basic_var_2}
\end{align}
Substituting from~\eqref{eqn:hon_bnd_1} in~\eqref{eqn:bnd_basic_var_2} above we obtain that
\begin{align}
    \E_{\z^t_i}\left(\norm{g^t_i} -  \E_{\z^t_i}\norm{g^t_i} \right)^2 \leq \sigma^2  ~ , \quad \forall i \in \H. \label{eqn:exp_bnd_var_1}
\end{align}
Substituting from~\eqref{eqn:exp_bnd_2} and~\eqref{eqn:exp_bnd_var_1} in~\eqref{eqn:lem_bnd_nu_1} proves the lemma.
\end{proof}


\subsection{Proof of~\eqref{eqn:rho} in Theorem~\ref{thm:conv}}
\label{sub:rho}
In this subsection, we prove~\eqref{eqn:rho}, i.e., 
\[\rho  = 1 - \left(n^2 + (n-f)^2 \mu^2 \right) \, \eta \left( \overline{\eta} - \eta \right) \in (0, \, 1)\]
where $0 < \eta < \overline{\eta}$. As $\eta \left( \overline{\eta} - \eta \right) > 0$, obviously, $\rho < 1$. 
As $\eta \in (0, \, \overline{\eta})$, upon substituting $\eta = \overline{\eta} - \delta$, where $\delta \in (0, \, \overline{\eta})$, in~\eqref{eqn:rho} we obtain that
\begin{align}
    & \rho = 1 - \left(n^2 + (n-f)^2 \mu^2 \right) \, (\overline{\eta} - \delta) \delta \label{eqn:rho_exp_1} \\
    & = 1 + \left(n^2 + (n-f)^2 \mu^2 \right) \delta^2 -  \overline{\eta} \left(n^2 + (n-f)^2 \mu^2 \right) \delta = 1 \nonumber \\
    & + \left(n^2 + (n-f)^2 \mu^2 \right) \left( \delta - \frac{\overline{\eta}}{2}\right)^2 - \frac{\overline{\eta}^2 \left(n^2 + (n-f)^2 \mu^2 \right)}{4}. \nonumber
\end{align}
As $\left( \delta - \frac{\overline{\eta}}{2}\right)^2  \geq 0$,~\eqref{eqn:rho_exp_1} implies that
\begin{align}
    \rho \geq 1 - \frac{\overline{\eta}^2 \left(n^2 + (n-f)^2 \mu^2 \right)}{4}. \label{eqn:rho_exp_2}
\end{align}
Substituting $\overline{\eta}$ from~\eqref{eqn:step_learn} in~\eqref{eqn:rho_exp_2} we obtain that
\begin{align}
    \rho \geq 1 - \frac{(2\lambda + \mu)^2 n^2 \alpha^2}{ \left(n^2 + (n-f)^2 \mu^2 \right)}. \label{eqn:rho_exp_2}
\end{align}
Recall, from~\eqref{eqn:alpha_learn}, that
\begin{align}
    \alpha = \frac{\lambda n - f(2 \lambda + \mu)}{n(2 \lambda + \mu)}. \label{eqn:alpha_alt}
\end{align}
Substituting from above in~\eqref{eqn:rho_exp_2} we obtain that 
\begin{align}
    \rho \geq 1 - \frac{\left(\lambda n - f(2 \lambda + \mu)\right)^2}{\left(n^2 + (n-f)^2 \mu^2 \right)} = 1 - \frac{\left((n-f) \lambda - f(\lambda + \mu)\right)^2}{\left(n^2 + (n-f)^2 \mu^2 \right)}.\label{eqn:rho_exp_3}
\end{align}
As $\alpha > 0$, $(n-f) \lambda - f(\lambda + \mu) > 0$. Thus, 
\begin{align}
    \left( (n-f) \lambda - f(\lambda + \mu) \right)^2 \leq (n-f)^2 \lambda^2. \label{eqn:alpha_ineq_3}
\end{align}
Substituting from~\eqref{eqn:alpha_ineq_3} in~\eqref{eqn:rho_exp_3} we obtain that
\begin{align}
    \rho \geq 1 - \frac{ (n-f)^2 \lambda^2}{\left(n^2 + (n-f)^2 \mu^2 \right)}.\label{eqn:rho_exp_4}
\end{align}

Now, consider a minimum point $w^*$ of the expected loss function $Q(w)$, and an arbitrary finite $w \in \R^d$. Note that $\nabla Q(w^*) = 0$. Thus, Assumption~\ref{asp:lip} implies that
\begin{align}
    \norm{\nabla Q(w)} \leq \mu \norm{w - w^*}. \label{eqn:asm_1}
\end{align}
Now, under Assumption~\ref{asp:str_cvx},
\begin{align}
    \iprod{w - w^*}{\nabla Q(w)} \geq \lambda \norm{w - w^*}^2. \label{eqn:asm_2}
\end{align}
Due to Cauchy-Schwartz inequality, $\iprod{w - w^*}{\nabla Q(w)} \leq \norm{w - w^*} \norm{\nabla Q(w)}$.
Thus,~\eqref{eqn:asm_2} implies that
\begin{align}
    \norm{Q(w)} \geq \lambda \norm{w - w^*}.\label{eqn:asm_3}
\end{align}
From~\eqref{eqn:asm_1} and~\eqref{eqn:asm_3} we obtain that $\lambda \leq \mu$. Upon using this inequality in~\eqref{eqn:rho_exp_4} we obtain that 
\begin{align}
    \rho \geq 1 - \frac{ (n-f)^2 \mu^2}{\left(n^2 + (n-f)^2 \mu^2 \right)}.\label{eqn:rho_exp_5}
\end{align}
As $(n-f)^2 \mu^2 < n^2 + (n-f)^2 \mu^2$,~\eqref{eqn:rho_exp_5} implies that $\rho > 0$.

\subsection{Proof of~\eqref{eqn:rate} in Theorem~\ref{thm:conv}}
\label{sub:rate}
In this section, we prove~\eqref{eqn:rate}, i.e, for all $t \geq 0$,
\begin{align*}
    \E_{t} \norm{w^{t+1} - w^*}^2 \leq \rho^{t+1} \norm{w^{0} - w^*}^2 + \left(\frac{1 - \rho^{t+1}}{ 1- \rho}\right) \M^2.
\end{align*}
~

Consider an arbitrary iteration $t$. 
Recall from~\eqref{eqn:order} in Algorithm~\ref{alg:cge} that the stochastic gradient with the $j$-th smallest norm, $g^t_{i_j}$, is sent by agent $i_j$ where $j \in \{1, \ldots, \, n\}$. Let 
\begin{align}
    \g^t = \sum_{j \in \{i_1, \ldots, \, i_{n-f}\}} g^t_{j} \label{eqn:notation_g}
\end{align}
denote the aggregate of the $n-f$ stochastic gradients received by the server with the $n-f$ smallest Euclidean norms. Upon substituting $\eta_t = \eta$, and $\g^t$ from~\eqref{eqn:notation_g}, in~\eqref{eqn:algo_1} we obtain that
\begin{align}
    w^{t+1} = w^t - \eta \, \g^t, \quad \forall t. \label{eqn:algo_1_2}
\end{align}
Thus, from the definition of Euclidean norm, 
\begin{align}
    \norm{w^{t+1} - w^*}^2 = \norm{w^{t} - w^*}^2 - 2 \eta \iprod{w^t - w^*}{ \g^t} + \eta^2 \norm{\g^t}^2. \label{eqn:no_exp_1}
\end{align}
Now, owing to the triangle inequality, we obtain that
\begin{align}
    \norm{\g^t} \leq \sum_{j \in \{i_1, \ldots, \, i_{n-f}\}} \norm{g^t_j}. \label{eqn:triangle_notation_g_0}
\end{align}
Recall that $\H$ denotes a set comprising $n-f$ non-faulty agents, i.e., $\mnorm{\H} = n-f$, and $\{i_1, \ldots, \, i_{n-f}\}$ represents the agents that sent stochastic gradients with smallest $n-f$ norm. Thus,
\begin{align}
    \sum_{j \in \{i_1, \ldots, \, i_{n-f}\}} \norm{g^t_j} \leq \sum_{j \in \H} \norm{g^t_j}. \label{eqn:honest_filter_sum}
\end{align}
Substituting from~\eqref{eqn:honest_filter_sum} in~\eqref{eqn:triangle_notation_g_0} we obtain that 
\begin{align}
    \norm{\g^t} \leq \sum_{j \in \H} \norm{g^t_j}. \label{eqn:triangle_notation_g}
\end{align}
Thus, $\norm{\g^t}^2 \leq \mnorm{\H} \sum_{j \in \H} \norm{g^t_j}^2 = (n-f) \sum_{j \in \H} \norm{g^t_j}^2$.
Substituting this in~\eqref{eqn:no_exp_1} implies that
\begin{align}
    \norm{w^{t+1} - w^*}^2 & \leq \norm{w^{t} - w^*}^2 - 2 \eta \iprod{w^t - w^*}{ \g^t} \nonumber \\
    & + \eta^2 \, (n-f) \sum_{j \in \H} \norm{g^t_j}^2. \label{eqn:no_exp_2}
\end{align}
Let $\H^t = \{i_1, \ldots, \, i_{n-f}\} \cap \H$ and $\B^t = \{i_1, \ldots, \, i_{n-f}\} \setminus \H^t$. Note that
\begin{align}
    \mnorm{\H^t} \geq \mnorm{\H}-f = n-2f, ~ \text{ and }~ \mnorm{\B^t} \leq f.\label{eqn:card_h}
\end{align}
Therefore, recalling from~\eqref{eqn:notation_g},
\begin{align}
    \g^t = \sum_{i \in \H^t}g^t_i + \sum_{j \in \B^t} g^t_j.
\end{align}
Thus,
\begin{align}
    \iprod{w^t - w^*}{ \g^t} = \sum_{i \in \H^t}\iprod{w^t - w^*}{g^t_i} + \sum_{j \in \B^t} \iprod{w^t - w^*}{ g^t_j }.  \label{eqn:no_exp_phi_1}
\end{align}
Owing to Cauchy-Schwartz inequality, $\forall j$,
\begin{align}
    \iprod{w^t - w^*}{g^t_j} \geq - \norm{w^t - w^*} \, \norm{g^t_j}. \label{eqn:no_exp_cs_1}
\end{align}
As in Lemma~\ref{lem:order_stat}, let $\nu_t$ denote the non-faulty agent in set $\H$ having stochastic gradient with the largest norm in iteration $t$. Thus, $\norm{g^t_{i_{n-f}}} \leq \norm{g^t_{\nu_t}}$, and   
\begin{align}
    \norm{g^t_j} \leq \norm{g^t_{\nu_t}}, \quad \forall j \in \B^t. \label{eqn:limit_fault}
\end{align}
Substituting from~\eqref{eqn:limit_fault} in~\eqref{eqn:no_exp_cs_1} we obtain that
\begin{align}
    \iprod{w^t - w^*}{g^t_j} \geq - \norm{w^t - w^*} \, \norm{g^t_{\nu_t}}. \label{eqn:no_exp_cs_2}
\end{align}
Upon substituting from~\eqref{eqn:no_exp_cs_2} in~\eqref{eqn:no_exp_phi_1} we obtain that
\begin{align*}
    \iprod{w^t - w^*}{ \g^t} \geq \sum_{i \in \H^t}\iprod{w^t - w^*}{g^t_i} - \sum_{j \in \B^t} \norm{w^t - w^*} \, \norm{g^t_{\nu_t}}. 
\end{align*}
As $\mnorm{\B^t} \leq f$ (see~\eqref{eqn:card_h}), from above we obtain that 
\begin{align}
    \iprod{w^t - w^*}{ \g^t} \geq \sum_{i \in \H^t}\iprod{w^t - w^*}{g^t_i} - f \norm{w^t - w^*} \, \norm{g^t_{\nu_t}}.  \label{eqn:no_exp_phi_3}
\end{align}
We define, for all $t$,
\begin{align}
    \phi_t = \sum_{i \in \H^t}\iprod{w^t - w^*}{g^t_i} - f \norm{w^t - w^*} \, \norm{g^t_{\nu_t}}. \label{eqn:new_phi_t}
\end{align}
Upon substituting from~\eqref{eqn:new_phi_t} in~\eqref{eqn:no_exp_phi_3} we get $\iprod{w^t - w^*}{ \g^t} \geq \phi_t$.
Substituting this in~\eqref{eqn:no_exp_2} we obtain that
\begin{align}
    \norm{w^{t+1} - w^*}^2 & \leq \norm{w^{t} - w^*}^2 - 2 \eta \, \phi_t \nonumber \\
    & + \eta^2 \, (n-f) \sum_{j \in \H} \norm{g^t_j}^2. \label{eqn:no_exp_3}
\end{align}
Recall the definition of random variable $\zeta^t$ from~\eqref{eqn:def_zeta}. Taking the expectation $\E_{\zeta^t}$ on both sides in~\eqref{eqn:no_exp_3}, and using the fact that $\E_{\zeta^t}\norm{w^{t} - w^*}^2 = \norm{w^{t} - w^*}^2$, implies that
\begin{align}
    \E_{\zeta^t}\norm{w^{t+1} - w^*}^2 & \leq \norm{w^{t} - w^*}^2 - 2 \eta \, \E_{\zeta^t} \left(\phi_t\right) \nonumber \\
    & + \eta^2 \, (n-f) \sum_{j \in \H} \E_{\zeta^t} \norm{g^t_j}^2. \label{eqn:no_exp_4}
\end{align}
Taking expectation $\E_{\zeta^t}$ on both sides in~\eqref{eqn:new_phi_t} implies that
\begin{align}
    \E_{\zeta^t} \left(\phi_t \right) = & \sum_{i \in \H^t}\iprod{w^t - w^*}{ \E_{\zeta^t}\left(g^t_i \right)} \nonumber \\
    & - f \norm{w^t - w^*} \, \E_{\zeta^t} \norm{g^t_{\nu_t}}. \label{eqn:exp_new_phi}
\end{align}
As $\E_{\zeta^t}\left(g^t_i \right)  = \nabla Q(w^t)$ for all $i \in \H$,~\eqref{eqn:exp_new_phi} implies that
\begin{align}
    \E_{\zeta^t} \left(\phi_t \right) & = \sum_{i \in \H^t}\iprod{w^t - w^*}{ \nabla Q(w^t)} \nonumber \\
    & - f \norm{w^t - w^*} \, \E_{\zeta^t} \norm{g^t_{\nu_t}}. \label{eqn:phi_2}
\end{align}
As $w^*$ is a minimum of $Q(w)$, $\nabla Q(w^*) = 0$. Thus, Assumption~\ref{asp:str_cvx}, i.e., strong convexity of function $Q(w)$, implies that
\begin{align}
    \iprod{w^t - w^*}{ \nabla Q(w^t)} \geq \lambda \norm{w^t - w^*}^2. \label{eqn:due_str_cvx}
\end{align}
Substituting from~\eqref{eqn:due_str_cvx} in~\eqref{eqn:phi_2} we obtain that
\begin{align}
    \E_{\zeta^t} \left(\phi_t \right) \geq \mnorm{\H^t} \, \lambda \norm{w^t - w^*}^2 - f \norm{w^t - w^*} \, \E_{\zeta^t} \norm{g^t_{\nu_t}}. \label{eqn:phi_3}
\end{align}
From Lemma~\ref{lem:order_stat}, $\E_{\zeta^t} \norm{g^t_{\nu_t}} \leq \sigma \left( 1 + \sqrt{n-f-1}\right) + \norm{\nabla Q(w^t)}$.
Substituting this in~\eqref{eqn:phi_3}, and using the fact that that $\mnorm{\H^t} \geq n-2f$ (see~\eqref{eqn:card_h}), we obtain that
\begin{align}
    \E_{\zeta^t} \left(\phi_t \right) & \geq \left( n \lambda - f(2 \lambda + \mu)\right) \norm{w^t - w^*}^2 \nonumber \\
    & - f \sigma\left( 1 + \sqrt{n-f-1}\right) \norm{w^t - w^*}. \label{eqn:main_bnd_phi}
\end{align}
Now, owing to Lemma~\ref{lem:bnd_norm}, $\E_{\zeta^t} \norm{g^t_j}^2 \leq \sigma^2 + \norm{\nabla Q(w^t)}^2$ for all $j \in \H$.
Recall that $\mnorm{\H} = n-f$. Thus,
\begin{align}
    \sum_{j \in \H} \E_{\zeta^t} \norm{g^t_j}^2 & \leq \mnorm{\H}\left( \sigma^2 + \norm{\nabla Q(w^t)}^2 \right) \nonumber \\
    & = (n-f)\left( \sigma^2 + \norm{\nabla Q(w^t)}^2 \right). \label{eqn:no_exp_sum_bnd_0}
\end{align}
As $\nabla Q(w^*) = 0$, Assumption~\ref{asp:lip} (i.e., Lipschitzness) implies that $\norm{\nabla Q(w^t)} \leq \mu \norm{w^t - w^*}$. Thus,~\eqref{eqn:no_exp_sum_bnd_0} implies that
\begin{align}
    \sum_{j \in \H} \E_{\zeta^t} \norm{g^t_j}^2 \leq  (n-f)\left( \sigma^2 + \mu^2 \norm{w^t - w^*}^2 \right). \label{eqn:no_exp_sum_bnd}
\end{align}
Finally, substituting from~\eqref{eqn:main_bnd_phi} and~\eqref{eqn:no_exp_sum_bnd} in~\eqref{eqn:no_exp_4} implies that
\begin{align}
    & \E_{\zeta^t} \norm{w^{t+1} - w^*}^2  \label{eqn:new_in} \\
    & \leq \left( 1 - 2 \eta \left( n \lambda - f(2 \lambda + \mu)\right) + \eta^2 (n-f)^2 \mu^2 \right) \norm{w^t - w^*}^2 \nonumber \\
    & + 2 \eta f \sigma\left( 1 + \sqrt{n-f-1}\right) \norm{w^t - w^*} + \eta^2 (n-f)^2 \sigma^2. \nonumber
\end{align}
For real values $a$ and $b$, $ 2 a b \leq a^2 + b^2$. Thus,
\begin{align}
    & 2 \eta f \sigma\left( 1 + \sqrt{n-f-1}\right) \norm{w^t - w^*} \leq \eta^2 n^2 \norm{w^t - w^*}^2 \nonumber \\
    & + \left( \frac{f}{n}\right)^2 \, \sigma^2\left( 1 + \sqrt{n-f-1}\right)^2. \label{eqn:basic_ineq_1}
\end{align}
Substituting from~\eqref{eqn:basic_ineq_1} in~\eqref{eqn:new_in} we obtain that
\begin{align*}
    & \E_{\zeta^t} \norm{w^{t+1} - w^*}^2 \leq \left\{ 1 - 2 \eta \left( n \lambda - f(2 \lambda + \mu)\right) \right\} \norm{w^t - w^*}^2 \\
    & + \eta^2 \left( n^2 + (n-f)^2 \mu^2 \right) \norm{w^t - w^*}^2 \\
    & + \left(\frac{f^2 \, \left( 1 + \sqrt{n-f-1}\right)^2}{n^2} + \eta^2 (n-f)^2 \right) \sigma^2. 
\end{align*}
Substituting $\M^2$ from~\eqref{eqn:ss_error} above we obtain that
\begin{align}
    & \E_{\zeta^t} \norm{w^{t+1} - w^*}^2 \leq \left( 1 - 2 \eta \left( n \lambda - f(2 \lambda + \mu)\right)\right) \norm{w^t - w^*}^2 \nonumber \\
    & + \eta^2 \left( n^2 + (n-f)^2 \mu^2 \right) \norm{w^t - w^*}^2 + \M^2. \label{eqn:growth_6}
\end{align}
Substituting $\alpha$ from~\eqref{eqn:alpha_learn}, we obtain that
\begin{align}
    n \lambda - f(2 \lambda + \mu) = (2 \lambda + \mu) n  \, \alpha. \label{eqn:alpha_morph}
\end{align}
Therefore, 
\begin{align*}
    & 2 \eta \left( n \lambda - f(2 \lambda + \mu)\right) - \eta^2 \left( n^2 + (n-f)^2 \mu^2 \right) \\
    & = 2 \eta (2 \lambda + \mu) n \alpha - \eta^2 \left( n^2 + (n-f)^2 \mu^2 \right) \\
    & = \left( n^2 + (n-f)^2 \mu^2 \right) \,  \eta  \left( \left(\frac{2 (2 \lambda + \mu) n}{n^2 + (n-f)^2 \mu^2} \right)\alpha - \eta \right).
\end{align*}
Substituting $\overline{\eta}$ from~\eqref{eqn:step_learn} above we obtain that
\begin{align}
    & 2 \eta \left( n \lambda - f(2 \lambda + \mu)\right) - \eta^2 \left( n^2 + (n-f)^2 \mu^2 \right) \nonumber \\
    & = \left( n^2 + (n-f)^2 \mu^2 \right) \eta \left( \overline{\eta} - \eta \right). \label{eqn:rate_1}
\end{align}
Substituting $\rho$ from~\eqref{eqn:rho} in~\eqref{eqn:rate_1} we obtain that
\begin{align}
    2 \eta \left( n \lambda - f(2 \lambda + \mu)\right) - \eta^2 \left( n^2 + (n-f)^2 \mu^2 \right) = 1 - \rho. \label{eqn:rate_2}
\end{align}
Substituting from~\eqref{eqn:rate_2} in~\eqref{eqn:growth_6} we obtain that
\begin{align}
    \E_{\zeta^t} \norm{w^{t+1} - w^*}^2 \leq \rho \norm{w^t - w^*}^2  + \M^2. \label{eqn:growth_7}
\end{align}
Recall from~\eqref{eqn:notation_exp} that $\E_0 = \E_{\zeta^0}$. Thus, the above proves the theorem for $t = 0$, i.e., 
\begin{align}
    \E_0 \norm{w^{1} - w^*}^2 \leq \rho \norm{w^0 - w^*}^2  + \M^2. \label{eqn:t_0}
\end{align}
Next, we consider the case when $t > 0$ in~\eqref{eqn:growth_7}. \\

From Section~\ref{sec:ft}, recall that the $w^t$ is a function of random variable $\zeta^{t-1} = \{\zeta^{t-1}_1, \ldots, \, \zeta^{t-1}_n\}$ given $w^{t-1}$. By retracing back to $t = 0$ we obtain that $w^t$ is a function of random variables $\zeta^0, \ldots, \, \zeta^{t-1}$, given the initial estimate $w^0$. As $w^{t+1}$ is a function of $w^t$ and $\zeta^t$, $\norm{w^{t+1} - w^*}^2$ is a function of random variables $\zeta^0, \ldots, \, \zeta^{t-1}$, given the initial estimate $w^0$. Let, for all $t > 0$,
\begin{align*}
    \E_{\zeta^t | \zeta^0, \ldots, \, \zeta^{t-1}} \norm{w^{t+1} - w^*}^2 
\end{align*}
denote the conditional expectation of $\norm{w^{t+1} - w^*}^2$ given the random variables $\zeta^0, \ldots, \, \zeta^{t-1}$ and $w^0$. Thus, for $t > 0$,
\begin{align}
    \E_{\zeta^t} \norm{w^{t+1} - w^*}^2 = \E_{\zeta^t | \zeta^0, \ldots, \, \zeta^{t-1}} \norm{w^{t+1} - w^*}^2. \label{eqn:growth_t}
\end{align}
Substituting from~\eqref{eqn:growth_t} in~\eqref{eqn:growth_7} we obtain that, given $w^0$,
\begin{align}
    \E_{\zeta^t | \zeta^0, \ldots, \, \zeta^{t-1}} \norm{w^{t+1} - w^*}^2 \leq \rho \norm{w^t - w^*}^2  + \M^2. \label{eqn:growth_t_1}
\end{align}
Now, note that due to Baye's rule, for all $t > 0$,
\begin{align*}
    & \E_{\zeta^0, \ldots, \, \zeta^{t}} \norm{w^{t+1} - w^*}^2 \\
    & = \E_{\zeta^0, \ldots, \, \zeta^{t-1}} \left( \E_{\zeta^t | \zeta^0, \ldots, \, \zeta^{t-1}} \norm{w^{t+1} - w^*}^2 \right).
\end{align*}
Substituting from~\eqref{eqn:growth_t_1} above implies that, given $w^0$,
\begin{align*}
    & \E_{\zeta^0, \ldots, \, \zeta^{t}} \norm{w^{t+1} - w^*}^2 \leq \E_{\zeta^0, \ldots, \, \zeta^{t-1}} \left(\rho \norm{w^t - w^*}^2  + \M^2\right) \nonumber \\
    & = \rho \,\E_{\zeta^0, \ldots, \, \zeta^{t-1}} \norm{w^t - w^*}^2 + \M^2, \quad \forall t > 0. 
\end{align*}
Recall from~\eqref{eqn:notation_exp} that notation $\E_t$ represents the joint expectation $\E_{\zeta^0, \ldots, \, \zeta^{t}}$ given $w^0$ for all $t$. Upon substituting this notation above we obtain that, for all $t > 0$,
\begin{align}
    \E_{t} \norm{w^{t+1} - w^*}^2 \leq \rho \,\E_{t-1} \norm{w^t - w^*}^2 + \M^2. \label{eqn:growth_t_2}
\end{align}
Finally, we use induction to show~\eqref{eqn:rate}, i.e., $\forall t \geq 0$,
\begin{align*}
    \E_{t} \norm{w^{t+1} - w^*}^2 \leq \rho^{t+1} \norm{w^{0} - w^*}^2 + \left(\frac{1 - \rho^{t+1}}{ 1- \rho}\right) \M^2.
\end{align*}

Recall that~\eqref{eqn:rate} trivially holds true for $t = 0$ due to~\eqref{eqn:t_0}. Assume that~\eqref{eqn:rate} holds true for $t = \tau-1$ where $\tau \geq 2$, i.e., 
\begin{align}
    \E_{\tau-1} \norm{w^{\tau} - w^*}^2 \leq \rho^{\tau} \norm{w^{0} - w^*}^2 + \left(\frac{1 - \rho^{\tau}}{ 1- \rho}\right) \M^2. \label{eqn:tau_ind_1}
\end{align}
From~\eqref{eqn:growth_t_2} we obtain that
\begin{align}
    \E_{\tau} \norm{w^{\tau+1} - w^*}^2 \leq \rho \,\E_{\tau-1} \norm{w^{\tau} - w^*}^2 + \M^2. \label{eqn:tau_ind_2}
\end{align}
Substituting from~\eqref{eqn:tau_ind_1} in~\eqref{eqn:tau_ind_2} above we obtain that
\begin{align*}
    \E_{\tau} \norm{w^{\tau+1} - w^*}^2 \leq \rho^{\tau + 1} \norm{w^{0} - w^*}^2 + \left(\frac{1 - \rho^{\tau+1}}{ 1- \rho} \right)\M^2.
    \label{eqn:tau_ind_3}
\end{align*}
The above proves~\eqref{eqn:rate} for $t = \tau$. Hence, due to reasoning by induction,~\eqref{eqn:rate} holds true for all $t \geq 0$.